\documentclass{llncs}
\usepackage[utf8]{inputenc}
\usepackage[T1]{fontenc}
\usepackage{hyperref}

\usepackage{amsmath,amssymb,mathtools}
\usepackage{tikz}
\usepackage{pgfplots}
\usepackage{pgfmath}
\usepackage{expl3}
\usepackage{xstring}
\usepackage{xparse}
\usepackage{etoolbox}
\usepackage{xcolor}
\usepackage{graphicx}
\usepackage{ifthen}
\usepackage{csquotes}
\usepackage{subcaption}
\usepackage{xspace}
\usepackage[textsize=small,disable]{todonotes}
\usepackage[shortcuts]{extdash}

\newenvironment{itemize*}{\vspace*{-.5em}\begin{itemize}\setlength{\itemsep}{0pt}}{\end{itemize}\vspace*{-.5em}}

\newcommand{\expect}[1]{\mathbb{E}[#1]}
\DeclareMathOperator{\varsym}{Var}
\newcommand{\var}[1]{\varsym(#1)}

\DeclareMathOperator{\traceSym}{tr}
\newcommand{\trace}[1]{\traceSym(#1)}

\newcommand{\ignoreLiteral}[1]{}

\newcommand{\refsec}[1]{Section~\ref{#1}}
\newcommand{\reffig}[1]{Fig.~\ref{#1}}

\newcommand{\nonDiagCMat}{\ensuremath{C}\xspace}
\newcommand{\CMat}{\ensuremath{C_1}\xspace}

\newcommand{\sID}[2][]{\ensuremath{\textrm{ID}#1\ifstrempty{#2}{}{_{\textrm{#2}}}}\xspace}
\newcommand{\ID}{\sID{}}
\newcommand{\sGED}{\sID{GED}}
\newcommand{\sMLE}{\sID{MLE}}
\newcommand{\sMOM}{\sID{M\kern-.5ptO\kern-.5ptM}}
\newcommand{\sALID}{\sID{A\kern-.5ptL\kern-.6ptI\kern-.5ptD}}
\newcommand{\sTLE}{\sID{TLE}}
\newcommand{\sGAbider}{\sID{A\kern-1ptB\kern-.5ptI\kern-.5ptD}}
\newcommand{\sSAbider}{\sID{R\kern-.5ptA\kern-1ptB\kern-.5ptI\kern-.5ptD}}

\newcommand{\GeomAbider}{\sGAbider}
\newcommand{\StrictAbider}{\sSAbider}

\newcommand{\sMSet}[1]{\texttt{m#1}\xspace}

\newcommand{\sMNIST}{\emph{MNIST}\xspace}
\newcommand{\sGisette}{\emph{Gisette}\xspace}

\pgfplotsset{compat=1.8}

\newcounter{tmparraycounter}
\newcommand{\arrayTmpSaveItem}[1]{%
    \stepcounter{tmparraycounter}%
    \expandafter\def\csname tmparrayname\thetmparraycounter\endcsname{#1}%
}
\newcommand{\arrayParse}[1]{%
    \setcounter{tmparraycounter}{0}%
    \renewcommand{\do}{\arrayTmpSaveItem}%
    \docsvlist{#1}%
}
\newcommand{\arrayElemNth}[2]{\arrayParse{#1}\csname tmparrayname#2\endcsname}

\def\arrayElemFirst#1{10}
\def\arrayElemLast#1{300}

\newcommand{\histAreaOpacity}{.5}
\newcommand{\histLineWidth}{1pt}
\newcommand{\histLegendScale}{.77}
\newcommand{\histLegendCols}{2}

\definecolor{orangeScheme1}{RGB}{242,200, 91}
\definecolor{orangeScheme2}{RGB}{251,164,101}
\definecolor{orangeScheme3}{RGB}{248,110, 81}
\definecolor{orangeScheme4}{RGB}{238, 62, 56}
\definecolor{orangeScheme5}{RGB}{209, 25, 62}

\definecolor{blueScheme1}{RGB}{ 83,204,236}
\definecolor{blueScheme2}{RGB}{ 25,116,211}
\definecolor{blueScheme3}{RGB}{  0,  1,129}

\definecolor{greenScheme1}{RGB}{204,255,204}
\definecolor{greenScheme2}{RGB}{179,230,185}
\definecolor{greenScheme3}{RGB}{153,204,166}
\definecolor{greenScheme4}{RGB}{128,179,147}
\definecolor{greenScheme5}{RGB}{102,153,128}
\definecolor{greenScheme6}{RGB}{ 77,128,108}
\definecolor{greenScheme7}{RGB}{ 51,102, 89}
\definecolor{greenScheme8}{RGB}{ 26, 77, 70}
\definecolor{greenScheme9}{RGB}{  0, 51, 51}

\definecolor{tolPalette1}{RGB}{ 51, 34,136}
\definecolor{tolPalette2}{RGB}{ 17,119, 51}
\definecolor{tolPalette3}{RGB}{ 68,170,153}
\definecolor{tolPalette4}{RGB}{136,204,238}
\definecolor{tolPalette5}{RGB}{221,204,119}
\definecolor{tolPalette6}{RGB}{204,102,119}
\definecolor{tolPalette7}{RGB}{170, 68,153}
\definecolor{tolPalette8}{RGB}{136, 34, 85}

\tikzset{
    histBarStyle/.style = {
        draw=#1,
        fill=#1,
        line width=\histLineWidth
    }
}

\definecolor{showcasePal1}{RGB}{ 51, 34,136}
\definecolor{showcasePal2}{RGB}{115,199,185}
\definecolor{showcasePal3}{RGB}{204,102,119}

\pgfplotscreateplotcyclelist{showcaseCycleList}{%
    {histBarStyle=showcasePal1},
    {histBarStyle=orangeScheme1},
    {histBarStyle=blueScheme1},
}

\pgfplotscreateplotcyclelist{estsCycleList}{%
    {histBarStyle=orangeScheme1},
    {histBarStyle=orangeScheme2},
    {histBarStyle=orangeScheme3},
    {histBarStyle=orangeScheme4},
    {histBarStyle=orangeScheme5},
    {opacity=0},
    {histBarStyle=blueScheme1},
    {histBarStyle=blueScheme2}
}

\pgfplotscreateplotcyclelist{nestedCycleList}{%
    {histBarStyle=tolPalette1},
    {histBarStyle=tolPalette2},
    {histBarStyle=tolPalette3},
    {histBarStyle=tolPalette4},
    {histBarStyle=tolPalette5},
    {histBarStyle=tolPalette6},
    {histBarStyle=tolPalette7},
    {histBarStyle=tolPalette8}
}

\ExplSyntaxOn
\newcommand{\histName}[1]{%
    \IfStrEqCase{#1}{%
        {GEDEstimatorFix}{\sGED}%
        {HillEstimator}{\sMLE}%
        {MOMEstimatorFix}{\sMOM}%
        {ALIDEstimatorFix}{\sALID}%
        {TightLIDEstimatorFix}{\sTLE}%
        {GeomAbider}{\sGAbider}%
        {StrictAbider}{\sSAbider}%
    }%
}
\ExplSyntaxOff

\newenvironment{histAreaAxis}[7][]{%
    \begin{axis}[%
        ybar,%
        xmin=#2,%
        xmax=#3,%
        ymin=0,%
        bar width=#4,%
        ymajorgrids,%
        scaled ticks=false,%
        y tick label style={%
            /pgf/number format/.cd,%
                fixed,%
                fixed zerofill,%
                precision=0,%
            /tikz/.cd%
        },%
        yticklabel={\pgfmathparse{\tick*100}\pgfmathprintnumber{\pgfmathresult}},%
        xtick pos=left,%
        ytick pos=left,%
        xlabel={\ID estimate},%
        ylabel={\% of points},%
        width=#6,%
        height=#7,%
        const plot mark mid,%
        every axis plot/.append style={%
            fill opacity=\histAreaOpacity,%
            draw opacity=0%
        },%
        cycle list name=#5,%
        #1%
    ]%
    \renewcommand{\addlegendentry}[1]{}%
}{%
    \end{axis}%
}
\newenvironment{histLineAxis}[7][]{%
    \begin{axis}[%
        xmin=#2,%
        xmax=#3,%
        ymin=0,%
        width=#6,%
        height=#7,%
        scaled ticks=false,%
        yticklabels={,,},%
        xticklabels={,,},%
        xtick pos=left,%
        ytick pos=left,%
        const plot mark mid,%
        every axis plot/.append style={%
            fill opacity=0%
        },%
        cycle list name=#5,%
        legend image post style={%
            fill opacity=\histAreaOpacity,%
            scale=\histLegendScale%
        },%
        legend style={%
            area legend,%
            at={(0,1)},%
            anchor=south west,%
            nodes={scale=\histLegendScale}%
        },%
        legend columns=\histLegendCols,%
        #1%
    ]%
}{%
    \end{axis}%
}

\definecolor{pcpGradCol1}{RGB}{211,148,  0}
\definecolor{pcpGradCol2}{RGB}{148,  0,211}
\definecolor{pcpGradCol3}{RGB}{  0,211,148}
\tikzset{
    pcpGradientA/.style={
        draw=pcpGradCol2!#1!pcpGradCol1
    },
    pcpGradientB/.style={
        draw=pcpGradCol3!#1!pcpGradCol2
    }
}

\newcommand{\lastplot}[3][]{}

\newcommand{\doubleIncludePlot}[1][]{%
    \def\plotWidthNoLegend{.45\textwidth}%
    \def\plotWidthWithLegend{.52\textwidth}%
    \def\plotWidth{170pt}%
    \def\plotHeight{140pt}%
    \def\plotNoLegend##1##2{%
        \begin{subfigure}[t]{\plotWidthNoLegend}%
            \includegraphics[#1]{##1}%
            \caption{##2}%
        \end{subfigure}%
    }%
    \def\plotWithLegend##1##2{%
        \begin{subfigure}[t]{\plotWidthWithLegend}%
            \includegraphics[#1]{##1}%
            \caption{##2}%
        \end{subfigure}%
    }%
}

\newcommand{\tripleIncludePlot}[1][]{%
    \def\plotWidthNoLegend{.33\textwidth}%
    \def\plotWidthWithLegend{.35\textwidth}%
    \def\plotWidth{125pt}%
    \def\plotHeight{100pt}%
    \def\plotYAxisOverride{}%
    \def\plotNoLegend##1##2{%
        \begin{subfigure}[t]{\plotWidthNoLegend}%
            \includegraphics[#1]{##1}%
            \caption{##2}%
        \end{subfigure}%
        \def\plotWidthNoLegend{.28\textwidth}%
        \def\plotYAxisOverride{ylabel={}}%
    }%
    \def\plotWithLegend##1##2{%
        \begin{subfigure}[t]{\plotWidthWithLegend}%
            \includegraphics[#1]{##1}%
            \caption{##2}%
        \end{subfigure}%
    }%
}

\pgfplotsset{%
    every axis/.append style={%
        label style={font=\scriptsize},%
        tick label style={font=\tiny},%
        x tick label style={yshift=.25em},%
        y tick label style={xshift=.1em},%
        ylabel shift=-5pt,%
        xlabel shift=-5pt,%
        legend style={%
            draw=none,%
            font=\footnotesize,%
            shift={(0pt,2pt)}%
        }%
    }%
}

\title{ABID: Angle Based Intrinsic Dimensionality}
\author{Erik Thordsen\orcidID{0000-0003-1639-3534} \and Erich Schubert\orcidID{0000-0001-9143-4880}} %
\institute{
TU Dortmund University, Dortmund, Germany
\texttt{\{erik.thordsen,erich.schubert\}@tu-dortmund.de}
}

\begin{document}
\hypersetup{pdfborderstyle={/S/U/W 0}}
\hypersetup{pdfborder={0 0 0}}%
\maketitle
\begin{abstract}
The intrinsic dimensionality refers to the ``true'' dimensionality of the data,
as opposed to the dimensionality of the data representation. For example, when
attributes are highly correlated, the intrinsic dimensionality can be much lower
than the number of variables. Local intrinsic dimensionality refers to the
observation that this property can vary for different parts of the data set;
and intrinsic dimensionality can serve as a proxy for the local difficulty of
the data set.

Most popular methods for estimating the local intrinsic dimensionality are
based on distances, and the rate at which the distances to the nearest
neighbors increase, a concept known as ``expansion dimension''. In this paper
we introduce an orthogonal concept, which does not use any distances:
we use the distribution of angles between neighbor points.
We derive the theoretical distribution of angles
and use this to construct an estimator for intrinsic dimensionality.

Experimentally, we verify that this measure behaves similarly,
but complementarily, to existing measures of intrinsic dimensionality.
By introducing a new idea of intrinsic dimensionality to the research community,
we hope to contribute to a better understanding of intrinsic dimensionality
and to spur new research in this direction.
\end{abstract}

\section{Introduction}

Intrinsic Dimensionality (ID) estimation is the process of estimating the dimension of a manifold embedding of a given data set either at each point of the data set individually or for the entire data set at large.
While describing the dimension of a given algebraic set at a specific point is a well-understood problem in algebra \cite{doi:10.1017/CBO9780511800474}, lifting these methods to a sample of an unknown function is not trivially possible.
Therefore methods that are very different from functional analysis are required to grasp the dimensionality of a discrete data set.
Prior work in the field is largely focused on analyzing the differential of  point counts in changing volumes \cite{DBLP:conf/kdd/AmsalegCFGHKN15,doi:10.1214/aos/1176343247,DBLP:conf/sisap/Houle17,DBLP:conf/icdm/HouleKN12}, as linear algebra gives estimates of these differentials assuming a certain dimensionality. These approaches rely on distances between points and assume the data to be uniformly sampled from their defining space. The resulting ID describes the dimension required to embed a point and its neighborhood in a manifold with reasonably small loss of precision.
In our novel approach for ID estimation, we derive an estimate based on the cosines between directional vectors of a point to all points in its neighborhood.
The basic idea is visualized in \reffig{fig:motivation}: in two-dimensional dense data, we see all directions evenly,
whereas in a linear subspace we mostly see parallel or opposite directions. Hence we aim at deriving an estimator
capable of computing the angles between observed data points.
It differs from the distance- and volume-based approaches as it describes the least dimensions required to connect a given point to the rest of the data set.
It can, therefore, be understood as a description of the simplicial composition of the data set.
Besides describing a different notion of local dimensionality,
we provide evidence that our approach is more robust and
gives stable estimates on smaller neighborhoods
than the volume-based approaches.
We hope that in the future the new angle-based interpretation of intrinsic dimensionality
will be combined with expansion-rate-based approaches
and spur further research in intrinsic dimensionality.

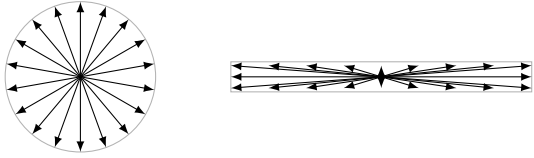
\begin{figure}[tb]\centering
\begin{tikzpicture}[scale=1]
\tikzstyle{l}=[->,>=latex];

\draw[black!30] (0,0) circle (1cm);
\foreach \a in {10,30,...,350}{
  \draw[l] (0,0) -- +(\a:1cm);
}

\draw[black!30] (2,-.2) rectangle (6,.2);

\draw[l] (4,0) -- +(+2,0);
\draw[l] (4,0) -- +(-2,0);
\foreach \a in {-2,-1.5,...,2}{
  \draw[l] (4,0) -- +(\a,+.15);
  \draw[l] (4,0) -- +(\a,-.15);
}
\end{tikzpicture}

\caption{Motivation of angle-based intrinsic dimensionality: in two-dimensional dense data,
we see all directions evenly, in noised one-dimensional linear data arrows go either in similar or in opposite directions.}
\label{fig:motivation}
\end{figure}

In \refsec{sec:related} we first discuss related work. We then discuss theoretical considerations
of dimensionality (and why it is not uniquely definable) in \refsec{sec:dimensionality}
and derive basic mathematical properties of angle distributions. %
In \refsec{sec:estimator} we introduce estimation techniques using this new angle-based notion of intrinsic dimensionality,
which we experimentally validate in \refsec{sec:experiments} before we conclude the paper in \refsec{sec:conclusions}.
\section{Related Work}\label{sec:related}

Intrinsic dimensionality has been shown to affect both the speed and accuracy
of similarity search problems such as approximate nearest-neighbor search
and the algorithms developed for this problem~\cite{DBLP:journals/jmlr/RadovanovicNI10,DBLP:conf/sisap/0001C19,DBLP:journals/ijait/BraticHKOR19}.
Intrinsic dimensionality has also been employed
to improve the quality of embeddings~\cite{DBLP:conf/sisap/SchubertG17},
to detect anomalies in data sets~\cite{DBLP:conf/sisap/HouleSZ18},
to determine relevant subspaces~\cite{DBLP:conf/sisap/BeckerHHLZ19},
and to improve generative adversarial networks (GANs)~\cite{DBLP:journals/corr/abs-1905-00643}.
Distance-based estimation of intrinsic dimensionality is the ``short tail'' equivalent
of extreme value theory~\cite{DBLP:conf/sisap/Houle17,DBLP:conf/sisap/Houle17a},
and many techniques can be adapted from estimators originally
devised for extreme values on the long tail of (censored-)
distributions~\cite{DBLP:journals/datamine/AmsalegCFGHKN18}, as previously used in disaster control.
Important estimators include the Hill estimator~\cite{doi:10.1214/aos/1176343247},
the aggregated version of it~\cite{doi:10.1198/073500101316970421},
the Generalized Expansion Dimension~\cite{DBLP:conf/icdm/HouleKN12},
method of moments estimators~\cite{DBLP:conf/kdd/AmsalegCFGHKN15},
regularly varying functions~\cite{DBLP:conf/kdd/AmsalegCFGHKN15},
and probability-weighted moments~\cite{DBLP:conf/kdd/AmsalegCFGHKN15}.
ELKI~\cite{DBLP:journals/corr/abs-1902-03616} also includes L-moments~\cite{doi:10.2307/2345653} based
adaptations of this and improves the bias of these estimators slightly.
A noteworthy recent development is the inclusion of pairwise distances
as additional measurements~\cite{tr/nii/ChellyHK16}
and the idea of also taking virtual mirror images of observed data points into account~\cite{DBLP:conf/sdm/AmsalegCHKRT19}.
We will note interesting parallels between these methods and our
new approach.

Angle-based approaches have been successfully used for outlier-detection in high-dimensional data,
for example with the method ABOD~\cite{DBLP:conf/kdd/KriegelSZ08}, which considers
points with a low variance of the (distance-weighed) angle spectrum to be anomalous,
with the assumption that such points are on the ``outside'' of the data set.
Our approach brings ideas from this method to
the estimation of intrinsic dimensionality (which in turn has been shown to
relate to outlierness~\cite{DBLP:conf/sisap/HouleSZ18}).
\section{On the Dimensionality of Functions and Data}\label{sec:dimensionality}

The dimensionality of a vector field in linear algebra is the
number of components of each vector; a quantity referred to as
\emph{representational dimensionality} because it characterizes the data
representation more than the underlying data. By selecting components of the
vectors, we can trivially obtain lower-dimensional projections.
Extending this to arbitrarily oriented linear projections gives
us affine subspaces also called \emph{linear manifolds}.
Yet, this is still not able to capture all varieties of dimensionality
that we use: consider the map $(x)\mapsto (x,\sin x)$, which clearly is a
(non-linear) embedding of a one-dimensional input space into a two-dimensional
representation. Because of this smooth map, and the ability to approximate the
resulting data to arbitrary precision with infinitesimal short linear pieces,
we consider such a curve to be a one-dimensional manifold.
This often aligns with human intuition,
for example when differentiating a circle (the outline)
from the corresponding disc (the contained area).
Yet, the mathematical notion of manifolds also has limitations: consider the
figure eight, which to a large extent resembles a line, just as the circle
-- except for the crossing point of the two lines, where a linear approximation
is no longer possible.
Similar problems will arise when we have to deal with a finite sample from the
data. For example, consider many parallel lines, close to each other. Analytically,
every sample will be from such a one-dimensional manifold. Yet, given only a finite
set of samples and close enough lines, the resulting data resembles a
two-dimensional plane.

The concept of \emph{local intrinsic dimensionality} (LID)~\cite{DBLP:conf/sisap/Houle17,DBLP:conf/sisap/Houle17a}
captures the need for allowing different parts of %
a data set to
have different dimensionality. Nevertheless, the \enquote{correct} answer to the question
of dimensionality is all but unambiguous: in the figure eight example, the data
is generated by a one-dimensional process, and also the expansion rate is linear,
but at the same time the crucial point cannot be locally approximated with a one-dimensional
linear function. A point on the surface of a ball (in $d$ dimensions) of uniform density lies on the
$(d{-}1)$-dimensional sphere surface, while points in the interior are $d$-dimensional -- which dimensionality should it be assigned?

In data analysis, we do not know the underlying functions. Sometimes we may
aim at selecting the best match from a given set of candidate functions,
but in many real problems we do not want to restrict ourselves to
such a candidate set, and we may not have enough data to become reasonably confident
to have found the \enquote{best} such function.
Hence, we opt for a non-parametric approach instead, where we attempt to estimate
the dimensionality based on the data samples; often centered around a particular
point of interest.

\begin{figure}[tb]
    \begin{subfigure}[t]{.25\linewidth}\centering
      \includegraphics{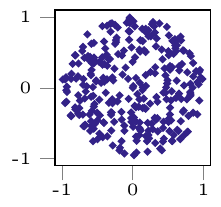}\\%
      \includegraphics{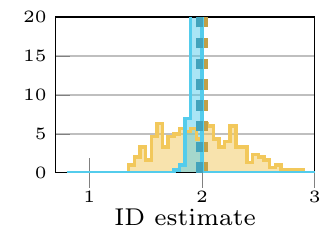}%
      \caption{Case 1}
    \end{subfigure}
    \hfill
    \begin{subfigure}[t]{.25\linewidth}\centering
      \includegraphics{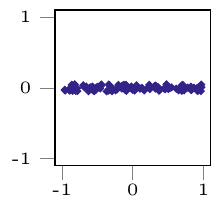}\\%
      \includegraphics{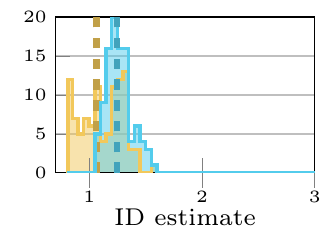}%
      \caption{Case 2}
    \end{subfigure}
    \hfill
    \begin{subfigure}[t]{.45\linewidth}\centering
      \includegraphics{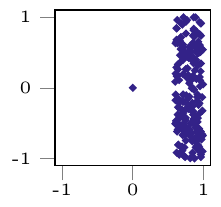}\\%
      \includegraphics{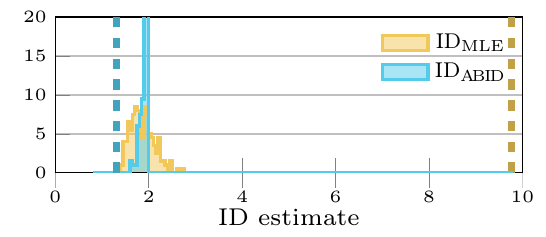}%
      \caption{Case 3}
    \end{subfigure}
    \caption{Three data sets with corresponding \ID estimate histograms.
    The dashed vertical lines correspond to the \ID estimate of the point at $(0,0)$.}
    \label{plot:showcase}
\end{figure}

Whilst existing work focuses on the analysis of distances in enclosing neighborhoods,
our novel approach utilizes the distribution of pairwise angles of neighboring points.
The different nature of the resulting \ID estimates is showcased in \reffig{plot:showcase}.
Both the distance\-/based (\sMLE, \cite{doi:10.1214/aos/1176343247}) and the angle\-/based (\sGAbider, this article)
approach consider Case~2 to be dominantly one-dimensional and Cases 1~and~3 to be mostly two-dimensional.
The outlying point $(0,0)$ in Case~3, however, is judged very differently by both approaches.
The distance\-/based approach \sMLE considers its environment to be almost ten-dimensional (${\approx}9.78$)
as all distances are very close, whereas our novel angle\-/based approach \sGAbider considers it to be one-dimensional (${\approx}1.18$) as the observed neighborhood lies in a narrow cone, similar to Case~2.
With increased neighborhood size this cone widens and the \ID gets closer to 2. This effect is similar to visual details of a surface disappearing at a distance.
In the example of a point on the surface of a $d$-ball mentioned above, the angle\-/based approach
is hence inclined to give an estimate below $d$ in contrast to distance\-/based approaches.

We will now lay the mathematical foundations
for our novel ID estimator.
When estimating the ID of a given point from a data set,
the general approach is to consider a number of nearby points
as an enclosing neighborhood.
An assumption shared by all ID estimators is that this
neighborhood should be \enquote{representative} of an underlying
manifold.
In approaches like GED \cite{DBLP:conf/icdm/HouleKN12}, \enquote{representative} can be
understood as uniformly sampled from the parameter space.
We hence assume that nearby points behave like a uniformly
random sampled $d$-dimensional subspace of the embedding
space, where $d$ is the dimension of some manifold, representing
the sampled parameter space.
When analyzing the neighborhood of some point $x$, we can
compute the angles between the directional vectors
$x_i{-}x$ for all points $x_i$ in the neighborhood
of $x$.
As the angles are independent of the lengths of these
vectors, we can without loss of generality assume that
they lie on a unit sphere.
We are therefore interested in the angles between points
sampled uniformly random from some unit $d$-sphere.
It is noteworthy that the unit $d$-sphere contains all
unit $d'$-spheres with $d'{\leq}d$ as a subset.
Hence, this assumption holds for any intrinsic dimensionality
equal to or lower than that of the embedding space
whenever the embedding is locally linear.
Assuming that all vectors describing the
neighborhood lie on a unit sphere, we can use the distribution
of pairwise angles provided by Cai, Fan, and Jiang~\cite{DBLP:journals/jmlr/CaiFJ13}.

\begin{theorem}[Distribution of random angles in a \boldmath$d$-sphere]
The distribution of angles $\theta$ between two
random points sampled independently and uniformly from a $d$-sphere
converges, as the number of samples goes to infinity, to
\begin{align}
P(\theta)
&= \frac{\Gamma(\tfrac{d}{2})}
	{\Gamma(\tfrac{1}{2})\Gamma(\tfrac{d-1}{2})}
	\cdot \sin(\theta)^{d-2}
\label{eq:angle_dist}
\end{align}
where $\Gamma$ is the gamma function
and $\theta$ is defined on $[0,\pi]$.
\end{theorem}
\begin{proof}
See Cai, Fan, and Jiang~\cite{DBLP:journals/jmlr/CaiFJ13} for a detailed proof.
\end{proof}

Because angles are invariant of the vector lengths, this also holds
for points sampled from a $d$-ball instead of the sphere as well as
other rotation invariant distributions such as spherical Gaussians,
as long as the origin point is not included in the data
(for which the angle is undefined). Note that such points at
distance 0 cause problems for most estimators of intrinsic
dimensionality and are commonly removed for such estimators, too.

As popularly known from the \textit{curse of dimensionality},
all angles tend to become approximately orthogonal as dimensionality approaches infinity.
This causes Eq.~\eqref{eq:angle_dist} to concentrate around~$\frac\pi2$~\cite{DBLP:journals/jmlr/CaiFJ13}.
The distribution above is unwieldy and expensive to compute (as we need to compute the arcus cosines).
We, therefore, prefer to work directly on the cosines.
By applying the Legendre duplication formula
and doing a change of variables, we obtain the distribution
of cosines.

\begin{theorem}[Distribution of cosine similarities of points in a \boldmath$d$-sphere]\label{theorem:cos-dist}
The distribution of pairwise cosine similarities $C$ between
random points sampled independently and uniformly from a $d$-sphere
is
\begin{align}
P(C)
&= \tfrac{1}{2} B(\tfrac{1+C}{2};\tfrac{d-1}{2},\tfrac{d-1}{2})
\end{align}
where $B(x;\alpha,\beta)$ is the beta distribution p.d.f.{}
and $C$ is defined on $[-1,1]$.
\end{theorem}

\newcommand\invbeta{\raisebox{4pt}{$\smash{\tfrac{1}{B(\tfrac{d-1}{2},\tfrac{d-1}{2})}}$}\rule[-5pt]{0pt}{17pt}}
\begin{proof}
For this proof, we modify %
the well known Legendre duplication formula:
\begin{align}
\Gamma(x)\Gamma(x+\tfrac{1}{2})
&= 2^{1-2x}\Gamma(\tfrac{1}{2})\Gamma(2x)
\notag\\
\frac{\Gamma(x+\tfrac{1}{2})}{\Gamma(x)\Gamma(\tfrac{1}{2})}
&= \frac{2^{1-2x}\Gamma(2x)}{\Gamma(x)^2}
= \frac{1}{B(x,x)} \cdot \tfrac{1}{2}^{2x-1}
\label{eq:legendre}
\shortintertext{
where $B(\cdot,\cdot)$ is the beta function.
By using this in
Eq.~\eqref{eq:angle_dist} for $x{=}\tfrac{d-1}2$, we obtain
}
P(\theta)
&= \invbeta
\cdot \left(\tfrac12{\sin(\theta)}\right)^{d-2}
\notag
\shortintertext{
We can now substitute $\theta$ with $\arccos(C)$ by a
change of variable:
}
P(C)
&= \invbeta
\cdot \left(\tfrac12{\sin(\arccos(C))}\right)^{d-2}
\cdot \left\vert \tfrac{\partial}{\partial C} \arccos(C) \right\vert
\notag\\
&= \invbeta
\cdot (\tfrac12{\sqrt{1-C^2}})^{d-2}
\cdot \tfrac{1}{\sqrt{1-C^2}}
\notag\\
&= \invbeta
\cdot (\tfrac{(1-C)(1+C)}{2 \cdot 2})^{\tfrac{d-2}{2}}
\cdot ((1-C)(1+C))^{-\tfrac{1}{2}}
\notag\\
&= \invbeta
\cdot (1-\tfrac{1+C}{2})^{\tfrac{d-1}{2}-1}
\cdot (\tfrac{1+C}{2})^{\tfrac{d-1}{2}-1}
\cdot \tfrac{1}{2}
\notag\\[1ex]
&= \tfrac{1}{2} B\left(\tfrac{1+C}{2};\tfrac{d-1}{2},\tfrac{d-1}{2}\right)
\notag
\end{align}
which is a beta distribution rescaled to the interval $[-1,1]$, on which $C$ is defined.
\end{proof}

Based on this, we can easily obtain the following helpful corollary:

\begin{corollary}\label{theorem:cor-var}
The average cosine similarity of two
random points sampled independently and uniformly from a $d$-ball
is given by
\vspace{-\abovedisplayskip}\vspace{3pt}%
\begin{align*}
\mathbb{E}[C] &= 0
\;.
\shortintertext{The variance and the non-central second moment are given by}
\var{C} &= \mathbb{E}[C^2] = \tfrac{1}{d}
\;.
\end{align*}
\end{corollary}

\begin{proof}
This follows immediately from the central moments of beta distributions.
By Theorem~\ref{theorem:cos-dist} we have
$\tfrac{1+C}{2}{\sim} B(\tfrac{d-1}{2}, \tfrac{d-1}{2})$.
This \emph{symmetric} beta distribution has a mean of $\tfrac12$, and hence
$\mathbb{E}[\nonDiagCMat]{=}0$.
The variance of this beta distribution given
$\smash{\alpha{=}\beta{=}\tfrac{d-1}{2}}$
is $\smash{\var{\tfrac{1+C}{2}}{=}\tfrac{1}{4d}}$,
and hence
$\smash{%
\expect{(\tfrac{1+C}{2}-\tfrac{1}{2})^2}
{=}\expect{\tfrac{C^2}{4}}
{=}\tfrac{1}{4d}}$.
Because the mean is 0, the variance and the second non-central moment agree trivially.
\end{proof}

\section{Estimating Intrinsic Dimensionality}\label{sec:estimator}

Based on this theoretical distribution of cosine similarities
in a $d$-ball,
we propose new estimators of intrinsic dimensionality
based on the method of moments.
Similar to other estimators, we assume the data sample comes
from the local neighborhood of a point; usually from a ball.
The first moment of Corollary~\ref{theorem:cor-var} cannot be
used for estimation because it does not depend on $d$.
Both the variance and the second non-central moment, however,
are suitable for estimating intrinsic dimensionality,
as they depend inversely on $d$.
This simple dependency stands in contrast to the expansion-rate
based approaches, which generally obtain an exponential
relation to the dimensionality, as the volume of a $d$-ball
has $d$ in its exponent.
With this simpler dependency on $d$, we hope
to obtain a more robust measure even with smaller neighborhood sizes (fewer samples);
and as we do not need to compute logarithms it can be computed
more efficiently.
But we still have two choices: we can either estimate using the
variance $\smash{\hat{d}{=}1/\var{\nonDiagCMat}}$ or using the non-central
second moment $\smash{\hat{d}{=}1/\mathbb{E}[\nonDiagCMat^2]}$, which only
agrees if $\smash{\mathbb{E}[\nonDiagCMat]{=}0}$ as expected for a uniform ball.

Consider the scenario of many points sampled
from a hyperplane, but the point of interest is not on this hyperplane.
The local neighborhood will then consist of
samples in a circular region on this plane. If we move the point of
interest away from the plane, the average cosine between
the samples tends to~1, and the variance to~0.
The variance\-/based estimate would hence tend to infinity.
The second non-central moment will,
as the average cosine tends to~1, also tend to~1,
and we estimate $d{\rightarrow} 1$. We argue that this is the
more appropriate estimate, as the data concentrates in a single
far away area.

Inspired by the work of Amsaleg et al.~\cite{DBLP:conf/sdm/AmsalegCHKRT19},
we investigated the idea of also considering the reflections of
all points with respect to the point of interest. Such a reflection
would cause the average cosine in this example to be 0, as every pair of points
can be matched to the pair with the second point reflected.
In the above example, we would obtain two opposite discs of points
and the resulting variance would tend to 1.
The estimates of
the variance\-/based estimator would thus agree with the non-central moment.
We can show that when adding reflected points,
the variance and the non-central second moment become equivalent
(which could serve as additional justification for the approach
of \cite{DBLP:conf/sdm/AmsalegCHKRT19}):
Since $c(x_i,-x_j){=}{-c(x_i,x_j)}{=}c(-x_i,x_j)$, taking reflections
into account simply means that we obtain two positive and two negative
copies of each cosine. The resulting average then is always exactly 0,
and hence $\var{\nonDiagCMat^\prime}
{=}\mathbb{E}[\nonDiagCMat^{\prime2}]-\mathbb{E}[\nonDiagCMat^\prime]^2
{=}\mathbb{E}[\nonDiagCMat^{\prime2}]
{=}\mathbb{E}[\nonDiagCMat^2]$.
We, therefore, do not further consider using such reflections of points,
besides their implicit presence in the non-central second moment.

Up to this point, we have been working with the limit cases of
distributions. In the following, we now change to working with
a fixed data sample of $k$ points, centered around a point of interest.
For simplicity, we assume that the data has been translated, such that the
point of interest is always at the origin, and that this point
(as well as any duplicates of it) has been removed from the sample.
We will now work with \emph{all pairwise} cosine similarities
in a $k{\times}k$ matrix denoted~\nonDiagCMat{}.
The diagonal of this matrix is usually excluded from computations.
We use the term \CMat{} when the ones on the diagonal are to be included.
By $\nonDiagCMat^2$, we denote the individual squaring of cosines.
The next theorem will use both a fixed sample and the matrix \CMat{} with the diagonal included.

\begin{theorem}[Upper bound]\label{theorem:fullMatUpper}
Let $X = \{x_1, \ldots, x_k\} \subset \mathbb{R}^D$ be
a sample from a $d$-dimensional subspace embedded in $\mathbb{R}^D$
for some $d \leq D$.
Formally, let $X$ contain at least $d$ linearly
independent vectors and let all $x_i$ be linear
combinations of a given set of $d$ orthonormal basis vectors.
Then the following inequality holds
\begin{align*}
\expect{\CMat^2}^{-1} \leq d
\;.
\end{align*}
\end{theorem}
\begin{proof}
Let $\tilde{X}$ be the $k{\times}d$ matrix obtained from $X$
by first performing a change of basis to the
given orthonormal basis of size $d$,
then normalizing each vector to unit length
to produce $\tilde{x}_i$.
Neither the change of basis (which is a rotation)
nor the posterior normalization affects the cosine similarities,
and we hence have
\begin{align}
c(\tilde{x}_i,\tilde{x}_j) = c(x_i,x_j)
\;.
\label{eq:cosChangeOfBasis}
\end{align}
It immediately follows that $\tilde{X}$ has a rank of $d$,
as we still have $d$ linearly independent vectors. 
The matrix $\tilde{\CMat} {=} \tilde{X}\tilde{X}^T$ then
contains entries of the form
$\langle \tilde{x}_i, \tilde{x}_j \rangle$.
As all $\tilde{x}_i$ are normalized, $\tilde{\CMat}$ is equal to
the cosine similarities.
Per Eq.~\eqref{eq:cosChangeOfBasis} it then follows that
$\tilde{\CMat}$ is exactly $\CMat$.
Because \CMat{} is a cosine similarity matrix, the diagonal
entries are all $1$, and we have $\trace{\CMat}{=}k$. 
Since $\tilde{X}$ is a $k{\times}d$ matrix with rank $d$,
we know that the rank of \CMat{} is $d$ as well.
Therefore \CMat{} has $d$~eigenvalues $\lambda_1, \ldots,
\lambda_d$ with $\sum_{i=1}^d \lambda_i {=} \trace{\CMat} {=} k$.
The sum of squared entries $\|\CMat\|_2^2 $
equals the sum of squared eigenvalues $\textstyle\sum\nolimits_{i=1}^d \lambda_i^2$
and is minimized if every eigenvalue equals~$\frac{k}{d}$,
which means we have the following lower bound:
\begin{align}
\|\CMat\|_2^2 &=
\textstyle\sum\nolimits_{i=1}^d \lambda_i^2
\geq d \cdot \left(\tfrac{k}{d}\right)^2
= \tfrac{k^2}{d}
\label{eq:bound}
\end{align}
and by taking the inverse we obtain the upper bound
$\expect{\CMat^2}^{-1} \leq d$.
\end{proof}

This is an upper bound for estimating the intrinsic dimensionality using \CMat{},
and we can use this to also obtain an upper bound for \nonDiagCMat{}.

\begin{corollary}\label{theorem:nonDiagUpper}
Let $X = \{x_1, \ldots, x_k\} \subset \mathbb{R}^D$ be
a sample from a $d$-dimensional subspace embedded in $\mathbb{R}^D$
as defined in Theorem~\ref{theorem:fullMatUpper}.
If $k>d$, then %
\begin{align}
\expect{\nonDiagCMat^2}^{-1} &\leq \tfrac{k-1}{k-d} \cdot d
\;.
\end{align}
\end{corollary}
\begin{proof}
Because the difference between \nonDiagCMat{} and \CMat{} is the diagonal of ones, Eq.~\eqref{eq:bound} yields
\begin{align*}
\|\nonDiagCMat\|_2^2 = \|\CMat\|_2^2 - k
&\geq \tfrac{k^2}{d} - k = \tfrac{k(k-d)}{d}
\shortintertext{and hence the average of the remaining $k^2{-}k$ cells is}
\expect{\nonDiagCMat^2}
&\geq
\tfrac{k-d}{k-1}\cdot\tfrac{1}{d}
\end{align*}
which is equivalent to the inequality above. For $k{=}d$ we obtain a trivial bound.
\end{proof}

The difference of including the diagonal or not vanishes for large enough~$k$.
One could attempt to regularize $\expect{\nonDiagCMat^2}$
with~$\frac{k-1}{k-d}$.
The major problem therein is that we do not know $d$ in advance.
To control the maximal overestimation of~$d$, a sufficiently
large neighborhood can be used to lower the margin of error.
For example, to bound $\expect{\nonDiagCMat^2}^{-1}\leq d {+} c$,
at least
$k \geq \tfrac{1}{c}d^2 {+} (1{-}\tfrac{1}{c})d$
neighbors are required.
For the bound $d{+}1$ ($c{=}1$), this means we require $k\geq d^2$ samples.

To solve this self-referential problem, we can also attempt an iterative
refinement. It turns out that the fixed point of this regularization
yields exactly the result we obtain by using \CMat{} instead of \nonDiagCMat{}.
Because using \CMat{} corresponds to using a regularized version and because it has a
very elegant upper bound, we base our method on this estimate:

\begin{definition}[ABID]
Given a data set $X = \{x_1, \ldots, x_n\} {\subset}
\mathbb{R}^D$, the regularized angle\-/based
intrinsic dimensionality estimator
for a point $x_i$ is:
\begin{align*}
\GeomAbider(x_i; k) := \expect{\CMat(B_k(x_i))^2}^{-1}
\end{align*}
where $B_k(x_i)$ are the directional vectors from $x_i$
to the $k$ nearest neighbors of $x_i$
and $\CMat(B_k(x_i))$ are the pairwise cosine similarities
within $B_k(x_i)$.
\end{definition}

By choosing the neighborhood of any point in the specified
set by the $k$ nearest neighbors, the measure is invariant
under scaling.
Analogously, one can instead define the neighborhood by
a maximum distance to the central point.
The sole restriction thereby is that the size of the
neighborhood has to be greater or equal to $d{+}2$ as for
any smaller neighborhood, the estimator does not need
to be properly regularized.
Since the error of the non-regularized estimate is limited
for any neighborhood with size quadratic in the intrinsic
dimension, we further introduce a non-regularized version
for comparative analysis.

\begin{definition}[RABID]
Given a data set $X = \{x_1, \ldots, x_n\} \subset
\mathbb{R}^D$, the \emph{raw} angle\-/based
intrinsic dimensionality estimator
for a point $x_i$ is defined as
\begin{align*}
\StrictAbider(x_i; k) :=
\expect{\nonDiagCMat(B_k(x_i))^2}^{-1}
\end{align*}
where $B_k(x_i)$ are the directional vectors from $x_i$
to the $k$ nearest neighbors of $x_i$
and $\nonDiagCMat(B_k(x_i))$ are the pairwise cosine similarities
of \emph{different} vectors in $B_k(x_i)$.
\end{definition}

Beware that this estimator can cause a division by zero if all $k$ vectors are pairwise orthogonal,
and can return values larger than $k$. In such cases, it is recommended to treat the estimate
as $k$, because the input vectors span a $k$ dimensional subspace.
Nevertheless, this estimator is likely unstable for small $k$, and for large $k$,
it converges to \GeomAbider.

To interpret the estimates by the new method,
it is important to consider the domain they operate on.
The angle\-/based measure is bounded by the spanning dimensionality of the point set.
While distributions of angles are usually distorted by non-linear transformations,
many transformations such as rotations will retain this bound.
Hence the bound may nevertheless apply---at least approximately---for many projections
of lower-dimensional manifolds in higher dimensional embeddings.
It is easy to see that
angle-preserving transformations do not affect our measure,
while distance-preserving transformations will not affect distance\-/based estimators.
Our new measure is less affected by local non-linear
contractions and expansions such as the decreasing density
on the outer parts of Gaussian distributions, but
it tends to estimate higher dimensionality than distance-based-approaches
when the transformations are non-linear.
We do not consider this to be a flaw, just a different design that
may or may not have advantages:
a~common assumption in many methods and applications like manifold
learning is to have locally linear transformations that preserve small neighborhoods,
which will then affect neither angles nor densities.
Our estimator, which can be seen as estimating how many
dimensions such a locally linear embedding needs to have,
is arguably very close to the idea of such applications.

\section{Evaluation}\label{sec:experiments}
In our comparative evaluation, we consider several ID estimators
on many standard evaluation data sets of both artificial and natural origin.
As measures of quality, we analyze the estimated
\ID's consistency both with expected values (for synthetic data)
and with each other (for natural data with no true value).
We will further inspect the stability of \ID estimates
for varying neighborhood sizes.
Depending on the density of data sets, approaches that
require a large neighborhood to stabilize, tend
to be inapplicable.

The histograms shown in this section are limited to a
region of interest in both $x$ and $y$ direction for interpretability.
Outside of the presented range along the $x$-axis,
the distributions always show a smooth drop to zero with no
further peaks but may have a long tail.

\subsection{Reference Estimators}
We compare \GeomAbider and \StrictAbider to
the Hill estimator \sMLE \cite{doi:10.1214/aos/1176343247},
the measure-of-moments-based estimator \sMOM \cite{DBLP:conf/kdd/AmsalegCFGHKN15},
the generalized expansion dimension \sGED \cite{DBLP:conf/icdm/HouleKN12},
the augmented local ID estimator \sALID \cite{tr/nii/ChellyHK16}
and its successor, the tight LID estimator \sTLE \cite{DBLP:conf/sdm/AmsalegCHKRT19}
using the implementations in the ELKI framework~\cite{DBLP:journals/corr/abs-1902-03616}.
All of these alternative estimators are based on the expansion rate.
The \sTLE~is supposed to reduce the necessary sample size
in the neighborhood to acquire a good estimate, yet in our
experiments tends to give higher estimates than the other
distance\-/based approaches.

\subsection{Dimensionality of Fractal Curves}
\begin{figure}[tb!]
\doubleIncludePlot
\plotNoLegend{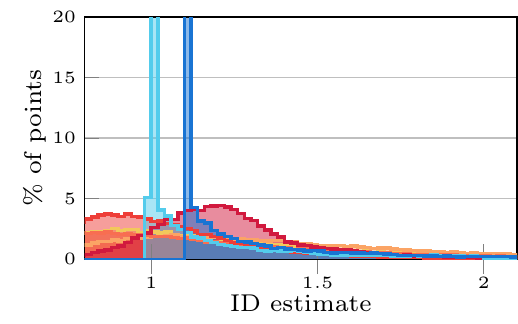}{10 neighbors}
\plotWithLegend{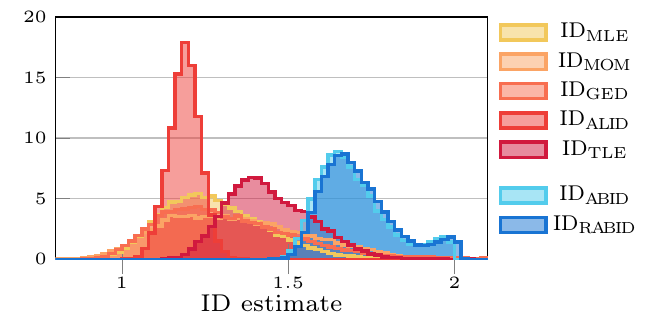}{200 neighbors}
\caption{Histograms of ID estimates of points sampled from
a Koch snowflake.}
\label{plot:Koch200}
\end{figure}
In line with the theoretical foundation of this work
and to demonstrate the different semantics of
angle\-/based and distance\-/based ID estimation, we
analyze the estimated ID of a well-known fractal, the Koch snowflake.
As seen in \reffig{plot:Koch200},
most distance\-/based approaches estimate a dimensionality
roughly around $\smash{\tfrac{\log 4}{\log 3}{\approx}1.26}$,
which is the Hausdorff dimension of the Koch snowflake,
when we consider enough neighbors ($k{=}200$).
This result is not surprising, as the distance\-/based
approaches are conceptually closely
related to classical fractal dimensions.
Our angle\-/based estimates, however, estimate a dimensionality of~${\approx} 1.6$
for larger neighborhoods, which is
larger than the fractal dimension, yet smaller than the representation dimension.
The difference can be explained by the
highly non-linear shape of the snowflake, as two consecutive
line segments are overlapping in a singularity.
Because the points are sampled from a Koch snowflake with finite recursion
depth, they must lie on this non-linear curve.
Reproducing the exact curve from a finite sampling, however,
is highly unlikely as specific parts of the embedding
space, that is $\mathbb{R}^2$, can be approximated to almost arbitrary precision.
The dimensionality of the sampling can, therefore, be
locally indistinguishable from the embedding space.
Where the distance\-/based approaches try to estimate the minimum
dimensionality that can possibly explain a model,
the angle\-/based approaches estimate the minimum
dimensionality from which a model is indistinguishable.
A higher estimate as parameter choice for downstream
applications, such as manifold learning, may turn out to
be more robust.
The results on further fractals,
such as the outline of $n$-flakes, were similar.

It is noteworthy that the scale of the neighborhood
has a large impact on the estimates.
When choosing a neighborhood small enough to mostly
stay within a line segment of the fractals (here $k{=}10$), the ID
estimates approximate $1$, as most neighborhoods
lie on straight lines.
For larger neighborhoods, the estimates approach
a proper representation of the manifold
space.
For too large neighborhoods, however, boundaries
of the point set as well as observing points distant
on the manifold, yet close in the embedding
space, tend to corrupt the estimates.

\subsection{Synthetic Data}
Amsaleg et al. \cite{DBLP:conf/sdm/AmsalegCHKRT19} used a
collection of synthetic and natural data sets, which
they provide for download.
\begin{figure}[tb!]
\doubleIncludePlot
\plotNoLegend{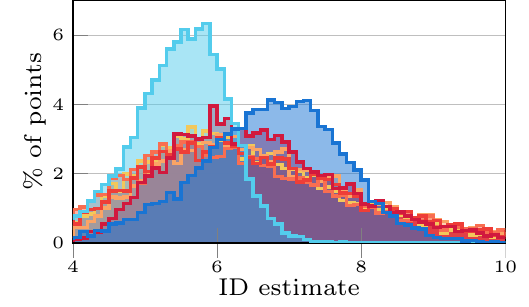}{30 neighbors}
\plotWithLegend{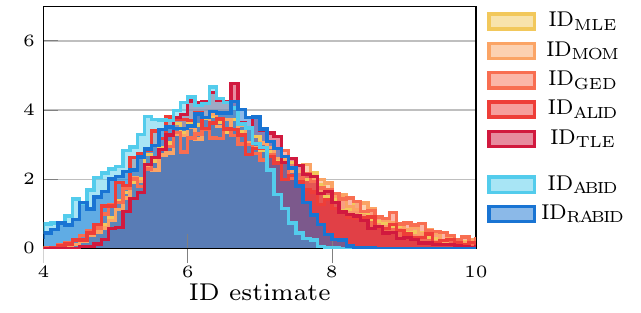}{150 neighbors}
\caption{Histograms of ID estimates of the \sMSet{6} set
with different neighborhood sizes.}
\label{plot:M6}
\end{figure}
\begin{figure}[tb!]
\begin{subfigure}[t]{.25\textwidth}
	\includegraphics{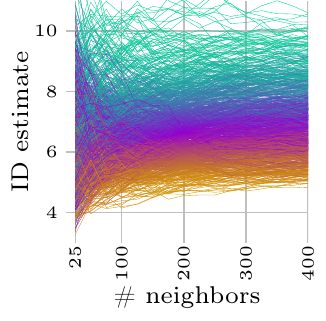}
	\caption{\sMLE}
\end{subfigure}
\begin{subfigure}[t]{.23\textwidth}
	\includegraphics{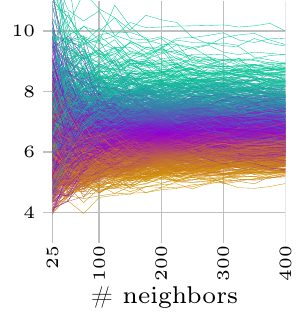}
	\caption{\sTLE}
\end{subfigure}
\begin{subfigure}[t]{.23\textwidth}
	\includegraphics{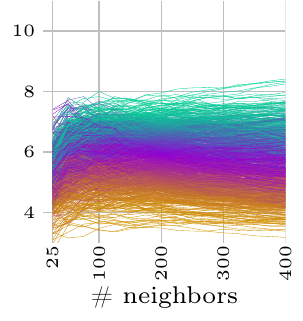}
	\caption{\sGAbider}
\end{subfigure}
\begin{subfigure}[t]{.23\textwidth}
	\includegraphics{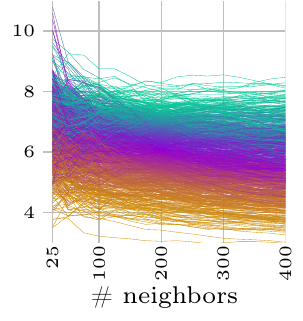}
	\caption{\sSAbider}
\end{subfigure}
\caption{Trails of estimates of 1000 points for varying neighborhood sizes on the \sMSet{6} set. Trail colors are assigned in order of \ID estimates at 200 neighbors.}
\label{plot:M6PCP}
\end{figure}
The \sMSet{6} data set consists of points sampled from a 6-dimensional
manifold non-linearly embedded in a 36-dimensional space.
As can be seen in \reffig{plot:M6}, for $k{=}150$ all estimators agree
on the data set to be inherently 6-dimensional at most points.
Where distance\-/based estimators tend to have a
long tail towards higher dimensions, the angle\-/based approaches
have an upper bound.
The estimates larger than 6 therefore must be artifacts from
the nonlinearity of the embedding used in this data set.
Even though this non-linearity shifts the upper bound beyond 6,
the angle\-/based approaches tend to have a shorter upper tail
and drop off faster to zero.
It is noteworthy that when comparing the estimates of the same method
at different neighborhood sizes between 30 and 300,
the angle\-/based approaches achieve higher scores than the distance\-/based approaches
on both Spearman's and Pearson's correlation coefficients.
In this sense, the angle\-/based approaches are more stable both in
the value of the estimates as well as the order of points by
estimated value when varying neighborhood sizes.
In more extreme neighborhoods (${<}30$ and ${>}300$), artifacts
from having too few samples for a reliable estimate and reaching the
boundaries of the data set, respectively, cause results to become less stable.
The stability is visualized in \reffig{plot:M6PCP} using trails of ID estimates for
individual points when varying the neighborhood size. In a perfectly stable result,
all lines would be parallel; instability causes lines that cross outside
their own color range (which represents the order at $k{=}200$) and hence the mixing of the colors.
The improved stability of the angle\-/based estimates is shown by a fairly
stable plot from 125 to 300 neighbors, whereas the distance\-/based estimates
begin to deviate much more at ${\leq}150$ and ${\geq}250$ neighbors respectively already.
Additionally, we can see in this plot that the average (the purple region)
of the distance\-/based estimates tends to increase with growing
neighborhood size whereas the distribution of \sGAbider appears
stable upwards of 100 neighbors. We can observe the upper bound property
of \sGAbider compared to \sSAbider.
The higher stability means that smaller neighborhoods suffice
for proper estimates and that the neighborhood parameter is
easier to choose.

\begin{figure}[tb]
\doubleIncludePlot
\plotNoLegend{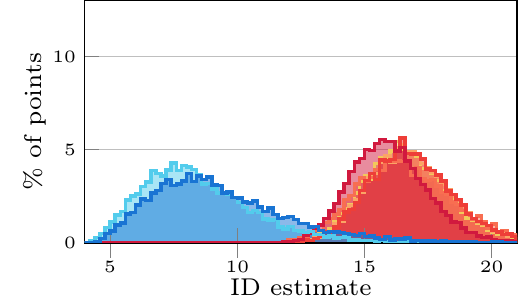}{100 neighbors}
\plotWithLegend{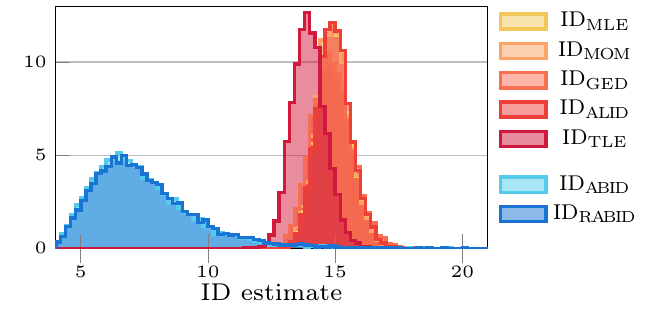}{500 neighbors}
\caption{Histograms of ID estimates of the \sMSet{10c} data set.}
\label{plot:M10C}
\end{figure}
The highest intrinsic dimensional data set provided by Amsaleg
et al. \cite{DBLP:conf/sdm/AmsalegCHKRT19}, \sMSet{10c}, is a 24-dimensional uniformly
sampled hypercube embedded in 25 dimensions.
On that data set, we observed a larger discrepancy between the
angle- and the distance\-/based approaches, shown in \reffig{plot:M10C}.
However, \sMSet{10c} consists of only 10,000 points,
which is the number of corners of a $\log_2(10,000){\approx} 13$ dimensional hypercube.
Hence, we doubt that this small sample can reliably represent a full
24-dimensional manifold, but the data likely is of much lower
dimensionality.
The estimates of the distance\-/based approaches move towards
this value as the neighborhood size increases.
Each of these 10,000 points then is, however,
essentially the corner of a 13-dimensional hypercube;
and will see the other data points as forming a hypercone,
producing smaller angles than if the data would evenly
surround the point. We believe it is because of this effect
(essentially a variant of the curse of dimensionality)
that the angle\-/based approaches estimate a far lower \ID.
\begin{figure}[tb]
\doubleIncludePlot
\plotNoLegend{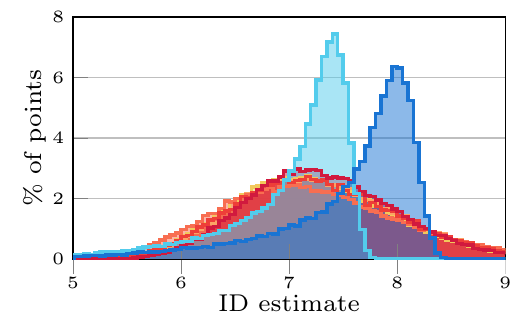}{100 neighbors}
\plotWithLegend{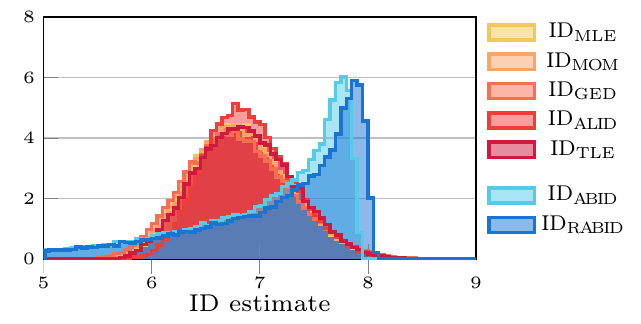}{500 neighbors}
\caption{Histograms of ID estimates of points on an 8-dimensional noisy
lattice.} %
\label{plot:lattice}
\end{figure}
To support this theory, we created a data set consisting of
points on the crossings of an 8-dimensional lattice, where
each dimension is sampled at the values
\todo[inline]{Schöner wäre 0,1,2,3 und Jitter maximal [0;1], eher kleiner.
Evtl. auch test mit zunehmendem Jitter.
Vielleicht für eine Langfassung oder deine Diss.\\
$\Rightarrow$ Okay, ja liest sich besser. Sollte für die Ergebnisse aber auch keinen Unterschied machen, wenn ich nichts übersehe. Der Jitter ist aus zwei Gründen so groß gewählt:\\
1. Soll damit die Regelmäßigkeit des Datensatzes stärker aufgelockert werden, damit es etwas näher an \textit{uniformly random} dran ist.\\
2. Die distanzbasierten Maße haben für kleinere obere Jittergrenzen wirklich nur Mist ausgegeben, weil zusätzliche Nachbarn immer in \enquote{Schüben} hinzukommen. Ab $\tfrac{1}{3}$ hatte sich das behoben. Ein Datensatz, der für die Referenzverfahren komplett unbrauchbar und worst case ist, empfand ich für eine vergleichende Analyse für zu unangemessen.}
$0, \frac{1}{3}, \frac{2}{3}$, and $1$ resulting in $4^8=65535$ points.
To smoothen the data
we added jitter to each point, uniformly drawn from $[0, \frac{1}{3}]^d$.
In that way, we obtained a data set that is more evenly distributed than uniform random
sampling and truly spans an 8-dimensional space.
On this data set, only the angle\-/based approaches were able
to estimate the correct dimension for most points as can
be seen in \reffig{plot:lattice}.
The many points where \sGAbider{} and \sSAbider{} estimate lower
values are likely the many points at the corners, edges, and sides of
this lattice.

\begin{figure}[t!]
\tripleIncludePlot
\plotNoLegend{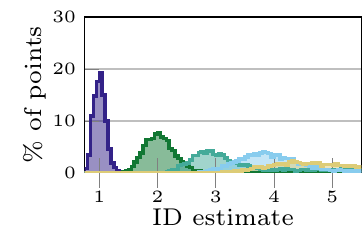}{\sMLE}
\plotNoLegend{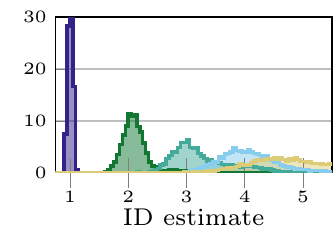}{\sTLE}
\plotWithLegend{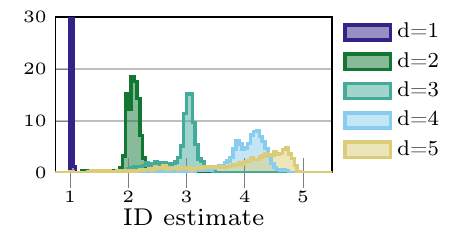}{\sGAbider}
\caption{Histograms of \ID estimates of nested hypercubes
with a neighborhood size of 100 colored by the hypercube
from which they were sampled.}
\label{plot:SNested5}
\end{figure}
To test the reliability of estimators in a mixture of manifolds,
we created instances of 1- through 5-dimensional hypercubes
linearly projected into the same 5-hypercube.
The projection was chosen such that every $d_i$-dimensional
hypercube intersects every $d_j$-dimensional hypercube in a
$\min(d_i,d_j)$-dimensional subspace.
For every hypercube, we sampled 5000 points uniformly at random
and computed \ID estimates using a neighborhood of different sizes.
In all experiments, the angle\-/based approaches were visibly
more capable of differentiating between the different dimensional
subspaces, which can be seen from the sharper spikes in
\reffig{plot:SNested5}.
Being capable of separating lower-dimensional subspaces
is an important feat, as noise in the embedding space can be
considered a high-dimensional manifold containing the
manifold of interest, and we believe this new \ID
estimate may help subspace discovery approaches that,
based on intrinsic dimensionality (e.g., \cite{DBLP:conf/sisap/BeckerHHLZ19}).
We observe that the angle\-/based approaches are more
robust against noise and in the presence of overlapping
subspaces.

\begin{figure}[tb!]
\doubleIncludePlot
\plotNoLegend{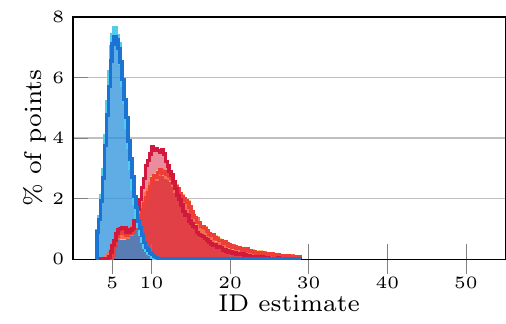}{\sMNIST}
\plotWithLegend{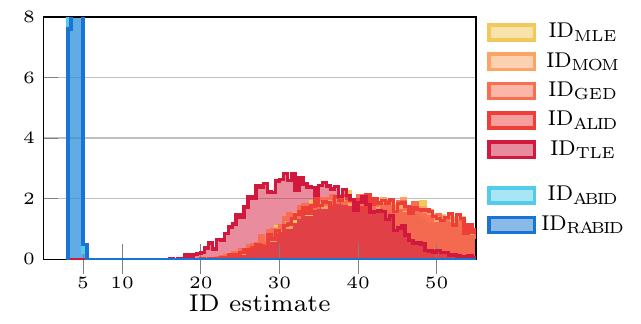}{\sGisette}
\caption{Histograms of ID estimates of the entire \sMNIST
data set and its noised subset \sGisette, both with
a neighborhood size of 300.}
\label{plot:MNIST}
\end{figure}
\subsection{Real Data}
We also analyzed the ID estimates on natural data sets.
The most interesting results were achieved on the \sMNIST data
set consisting of gray-scale images of hand-written digits
with a $28{\times}28$ resolution.
Our proposed approach estimates an ID of about 6 for most points,
whereas the distance\-/based approaches peak around 10 to 11.
From neighborhood sizes of 100 upwards, the distance\-/based
approaches, however, start forming a second peak at the same
ID as the angle\-/based approaches,
visible in \reffig{plot:MNIST}.
A possible explanation could be that the \sMNIST data set is
not uniformly random on the manifold, whereby small
environments are too noisy for distance\-/based approaches.
That claim is further supported by a higher resolution
subset of \sMNIST, \sGisette, consisting only of handwritten
$4s$ and $9s$ with a $50{\times}50$ resolution.
The 2500-dimensional image vectors in \sGisette were further
expanded by just as many uniformly random dimensions, mimicking
high-dimensional noise in a 5000-dimensional embedding.
The added noise harshly increased the estimates of the distance\-/based approaches.
The angle\-/based approaches, however, estimate a slightly lower
ID of about 4 for most points.
When considering a subset of points, the ID might vary in
both directions.
When retaining only a surface of a hypercube, the ID should
clearly reduce.
When sparsening the data such that, e.g., a space-filling curve
is reduced to a lattice-like data set, the ID estimate can
also increase.
However, we consider the smaller difference of the angle\-/based
estimates as more plausible,
even though the proposed estimates for the \sGisette data set
might be slightly too low as the high-dimensional noise
might have sparsened the local neighborhoods too much.
Nevertheless, we observe that the angle\-/based approach can
be more robust against noise in such a semi-real scenario.

\subsection{Estimator Interactions}
These experiments can also give some insight into the differences
and interactions of the different estimators.
As expected from theory, \sGAbider{} and \sSAbider{} converge towards
the same value for sufficiently large neighborhood sizes. Because it
is trivial to compute both estimators at once, we can use the
difference of the estimates to assess the quality; if they differ much
we may need larger sample sizes, if they are close the sample size
should be sufficient for this dimensionality. Because the angle\-/based
estimators appear to require fewer samples than the distance\-/based
approaches, this may also help to choose parameters for these methods.
Secondly, if the angle\-/based estimates are much smaller
than the distance\-/based estimates, the data set might not be
sufficiently densely sampled for this dimensionality;
if the angle\-/based estimates are much larger
than the distance\-/based estimates, the embedding may be
highly non-linear (as in the Koch snowflake example),
or may not preserve local density.

\section{Conclusions}\label{sec:conclusions}
In this paper, we propose a novel approach to estimate
local intrinsic dimensionality, along with two estimators,
\sGAbider and \sSAbider.
Instead of analyzing the expansion rate, as previous
distance\-/based approaches do, the novel approach focuses
on the geometry characterized by pairwise angles.
We have given an a~priori derivation of the novel estimators
derived from integral geometry.
Our experimental evaluation suggests that the novel approach
may be more robust against noise, computes a bit stabler estimates,
and gives estimates as reasonable as distance\-/based estimators,
albeit of a different nature.
We have further discussed how the difference between estimates
can hint at particular effects in the data.
The presented approach does not yet fully utilize all
interactions of the pairwise angles within a neighborhood,
which could lead to an improved \ID estimation in future work
by incorporating ideas of \cite{DBLP:conf/sdm/AmsalegCHKRT19}.
Future work may also investigate using
higher-order moments, as well as robust estimation techniques for
the second moment, such as the median average deviation or L-moments~\cite{doi:10.2307/2345653}.
As $\mathbb{E}[C]\gg 0$ indicates points remote from their neighbors,
this can be interesting to integrate into an outlier detection method based on
intrinsic dimensionality, which would yield a hybrid of ABOD~\cite{DBLP:conf/kdd/KriegelSZ08}
and LID outlier detection~\cite{DBLP:conf/sisap/HouleSZ18}.

\vfill
\pagebreak
\bibliographystyle{splncs04} %
\bibliography{literature}
\end{document}